\documentclass[letterpaper, 10 pt, conference]{ieeeconf}

\IEEEoverridecommandlockouts   
\usepackage[binary-units=true]{siunitx}
\usepackage{times}
\usepackage{amsmath,amssymb}
\usepackage{accents}
\usepackage{graphicx}
\usepackage{bm}
\let\labelindent\relax
\usepackage[inline]{enumitem}
\usepackage{array}
\usepackage{multicol}
\usepackage{multirow}
\usepackage{tabulary}
\usepackage{longtable}
\usepackage{booktabs}
\usepackage{caption}
\usepackage{subcaption}

\usepackage{todonotes}
\usepackage{cite}
\usepackage[hidelinks]{hyperref}
\usepackage{algorithm}
\usepackage{algpseudocode}
\usepackage{breqn}
\usepackage{tcolorbox}
\usepackage{mathtools}
\usepackage{tikz}
\usetikzlibrary{arrows}
\usetikzlibrary{calc}
\usetikzlibrary{arrows,decorations.markings}
\usepackage{nomencl}
\usepackage{balance}

\usetikzlibrary{shadows,trees,calc,arrows,shapes,positioning}
\allowdisplaybreaks

\definecolor{light-gray}{gray}{0.95}
\definecolor{dark-gray}{gray}{0.5}
\definecolor{mygray}{gray}{0.75}
\newcommand{\BIN}{\begin{bmatrix}}
\newcommand{\BOUT}{\end{bmatrix}}

\newtheorem{remark}{Remark}

\pgfdeclarelayer{bg1}   
\pgfdeclarelayer{bg}  
\pgfsetlayers{bg1,bg,main}  

\definecolor{orange}{rgb}{0.99,0.69,0.07}
\definecolor{lightgray}{gray}{0.85}
\definecolor{light-gray}{gray}{0.95}
\definecolor{dark-gray}{gray}{0.5}

\tikzset{cross/.style={cross out, draw=black, minimum size=2*(#1-\pgflinewidth), inner sep=0pt, outer sep=0pt},
cross/.default={1pt}}

 \makeatletter
 \newcommand\fs@spaceruled{\def\@fs@cfont{\bfseries}\let\@fs@capt\floatc@ruled
   \def\@fs@pre{\vspace{5pt}\hrule height.8pt depth0pt \kern2pt}%
   \def\@fs@post{\kern2pt\hrule\relax}%
   \def\@fs@mid{\kern2pt\hrule\kern2pt}%
   \let\@fs@iftopcapt\iftrue}
 \makeatother

\newtheorem{theorem}{Theorem}
\newtheorem{definition}{Definition}[section]

\usepackage{tikz-qtree}
\usetikzlibrary{shadows,trees,calc,arrows,shapes,positioning}
\algnewcommand{\IIf}[1]{\State\algorithmicif\ #1\ \algorithmicthen}
\algnewcommand{\EndIIf}{\unskip\ \algorithmicend\ \algorithmicif}

\title{\LARGE \bf Monte-Carlo Tree Search with Prioritized Node Expansion for Multi-Goal Task Planning}
\author{Kai Pfeiffer, Leonardo Edgar, Quang-Cuong Pham%
	\thanks{The authors are with the Schaeffler Hub for Advanced Research at NTU and School of Mechanical and Aerospace Engineering, Nanyang Technological University, Singapore}%
	\thanks{This work is partly supported by the Schaeffler Hub for Advanced Research at NTU, under the ASTAR IAF-ICP Programme ICP1900093.}%
}

\begin{document}
	\maketitle 
	\thispagestyle{empty}
	\pagestyle{empty}
	
	\begin{abstract}
		Symbolic task planning for robots is computationally challenging due to the combinatorial complexity of the possible action space. This fact is amplified if there are several sub-goals to be achieved due to the increased length of the action sequences.
		In this work, we propose a multi-goal symbolic task planner for deterministic decision processes based on Monte Carlo Tree Search. We augment the algorithm by prioritized node expansion which prioritizes nodes that already have fulfilled some sub-goals. 
		Due to its linear complexity in the number of sub-goals, our algorithm is able to identify symbolic action sequences of 145 elements to reach the desired goal state with up to 48 sub-goals while the search tree is limited to under 6500 nodes. 
		We use action reduction based on a kinematic reachability criterion to further ease computational complexity.
		We combine our algorithm with  object localization and motion planning and apply it to a real-robot demonstration with two manipulators in an industrial bearing inspection setting.
	\end{abstract}
	
	\section{Introduction} 
	
	Robots have drawn much attention in industrial settings, for example for de-caking of 3D-printed parts~\cite{Huy2020}, airplane assembly~\cite{Pfeiffer2017}, and construction~\cite{Lim2021}.
	In this work, we want to consider visual inspection, specifically of bearings. Due to the repetitive nature of the task and the need of handling high volumes of heavy bearings, this task would benefit from automation in terms of ergonomy, quality, and reliability. 
	This problem constitutes a multi-goal task and motion planning problem as multiple goals (inspection of several object faces, discarding, ...) need to be achieved for several objects. 
	The presence of several goals makes task planning significantly more difficult as long and complicated task plans with multiple robots have to be identified. To address this difficulty, we propose a fast depth-first variation of Monte Carlo Tree Search (MCTS)~\cite{Coulom206}  for deterministic decision processes based on prioritized node expansion (PNE). It is able to efficiently plan long action sequences for such multi-goal scenarios. A kinematic reachability criterion is used during the tree search to avoid kinematically infeasible logic sequences. A selection from the logic action sequence of length 35 for the inspection of two bearings with two UR3e robots constituting 12 sub-goals is given in Fig.~\ref{fig:key_trafo}.
	
	\begin{figure}[tp!]
		\vspace{6pt}
		\begin{subfigure}{0.49\columnwidth}
			\includegraphics[width=\linewidth]{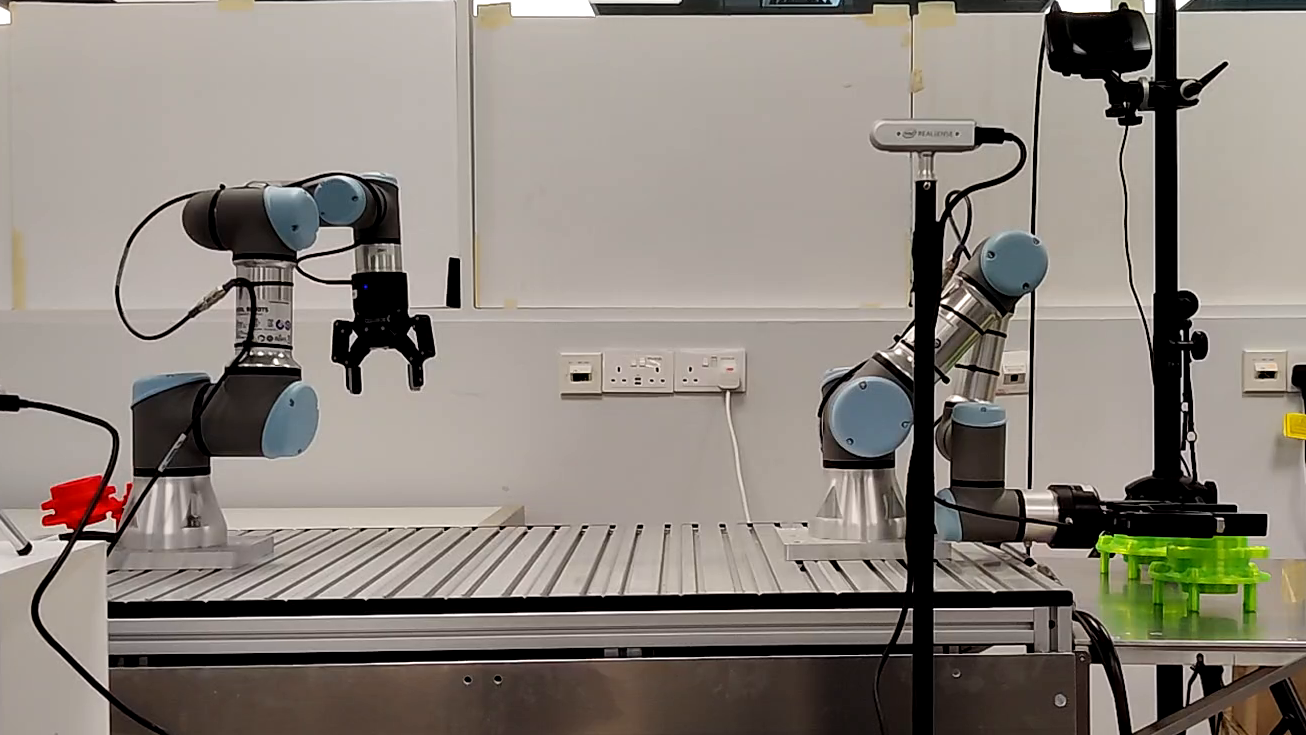}
			\caption{Side pick} \label{fig:side_pick}
		\end{subfigure}%
		\hspace*{\fill}   
		\begin{subfigure}{0.49\columnwidth}
			\includegraphics[width=\linewidth]{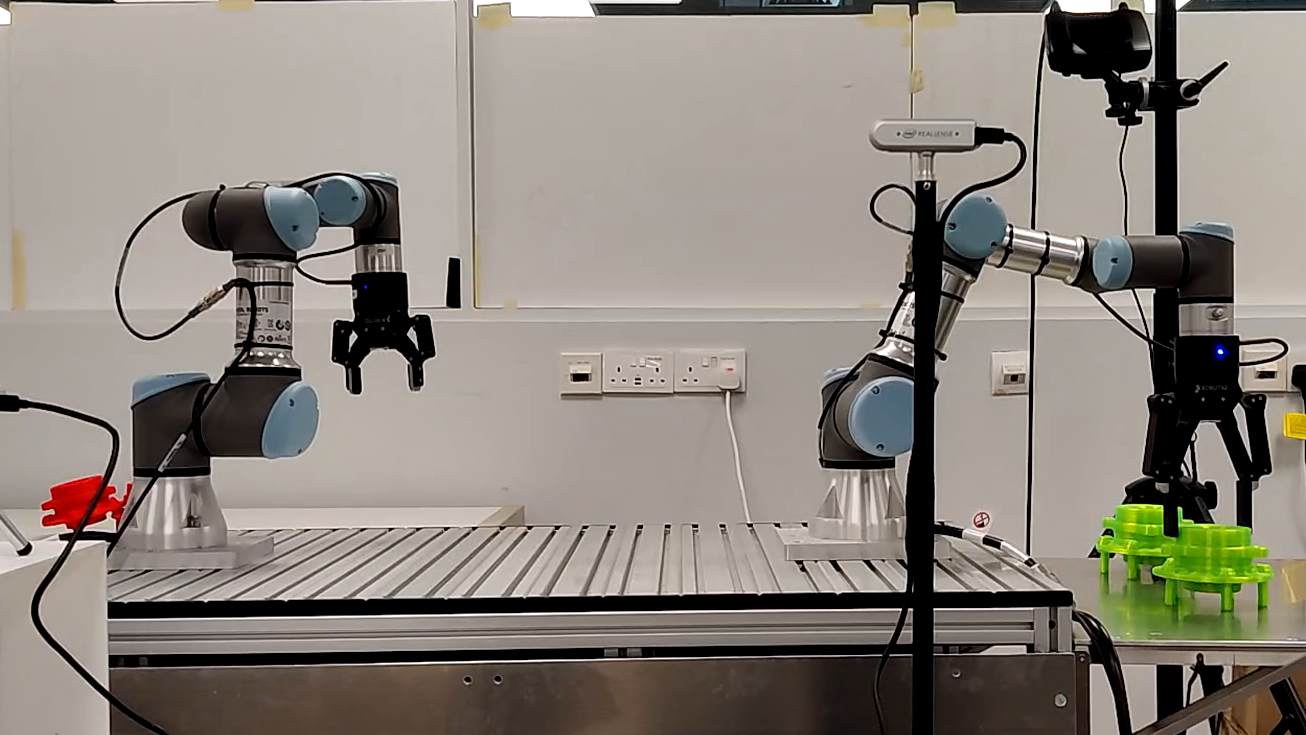}
			\caption{Top pick} \label{fig:top_pick}
		\end{subfigure}   
		
		\begin{subfigure}{0.49\columnwidth}
			\includegraphics[width=\linewidth]{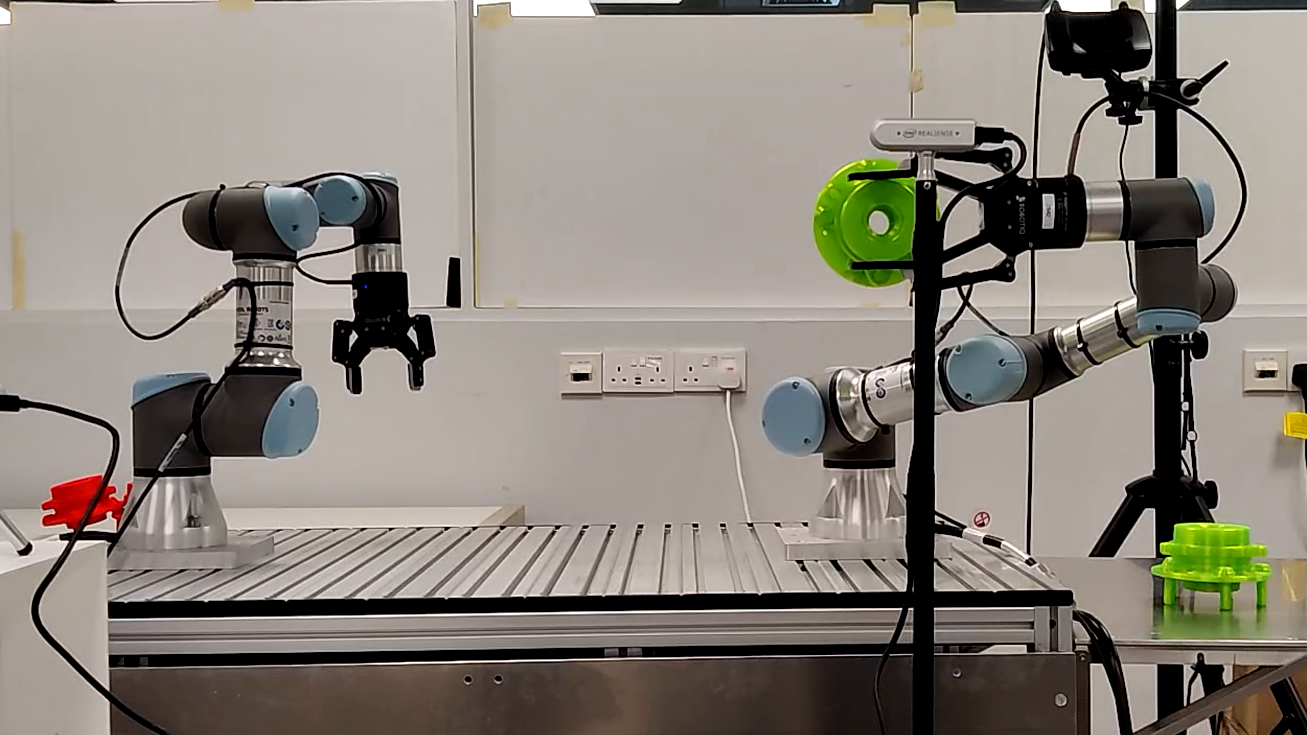}
			\caption{Present top} \label{fig:present_top}
		\end{subfigure}%
		\hspace*{\fill}   
		\begin{subfigure}{0.49\columnwidth}
			\includegraphics[width=\linewidth]{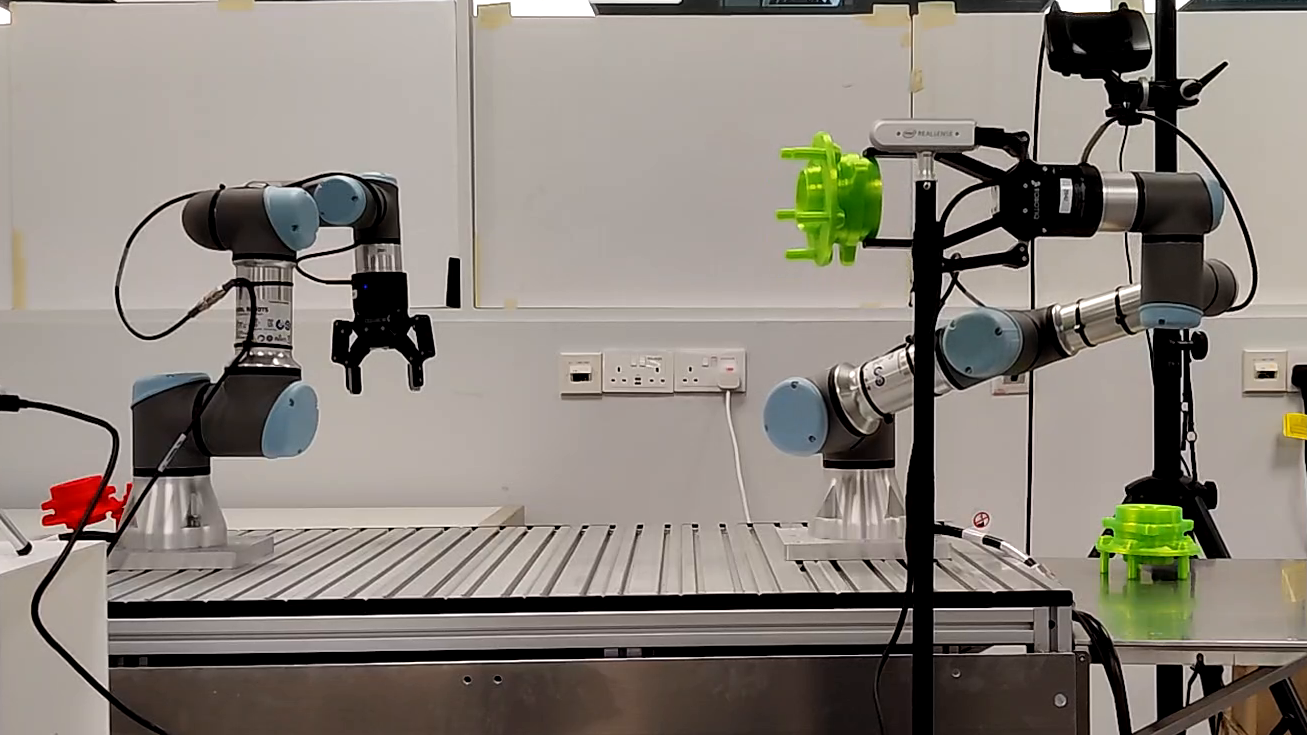}
			\caption{Present side} \label{fig:present_side}
		\end{subfigure}%
		
		\begin{subfigure}{0.49\columnwidth}
			\includegraphics[width=\linewidth]{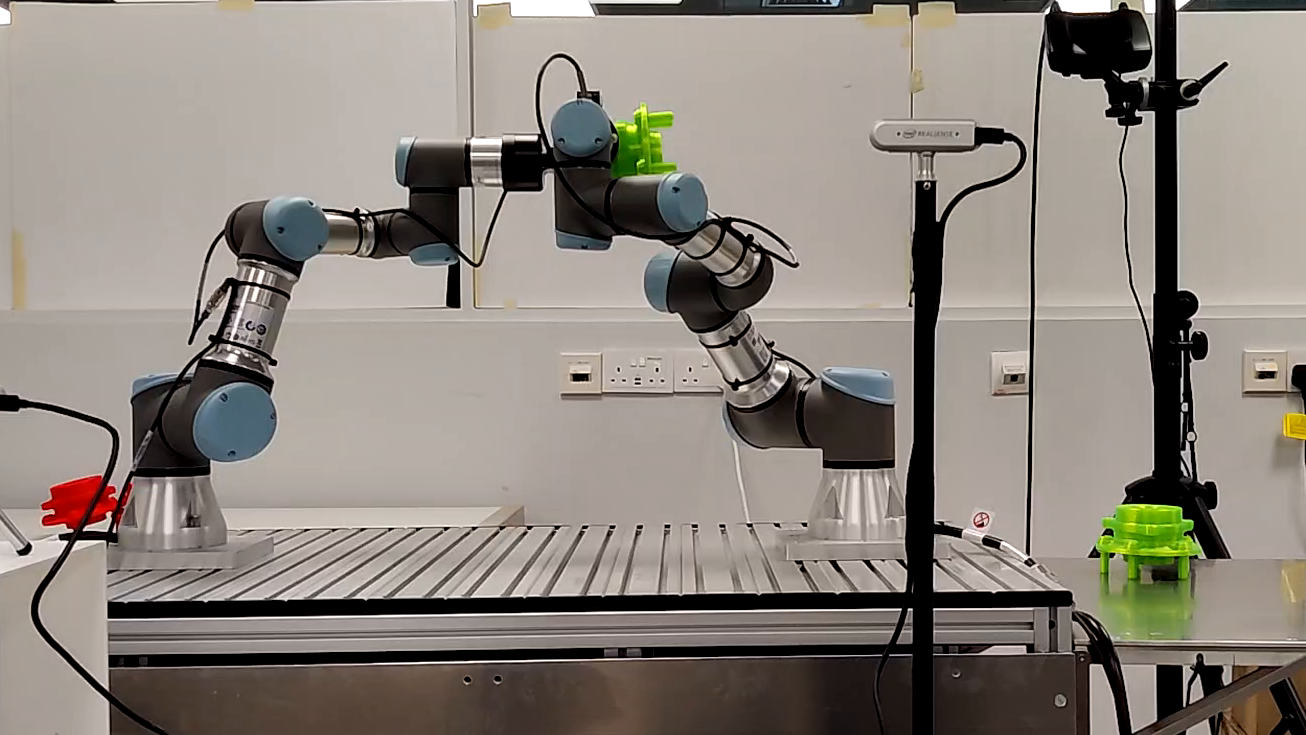}
			\caption{Handover} \label{fig:handover}
		\end{subfigure}%
		\hspace*{\fill}   
		\begin{subfigure}{0.49\columnwidth}
			\includegraphics[width=\linewidth]{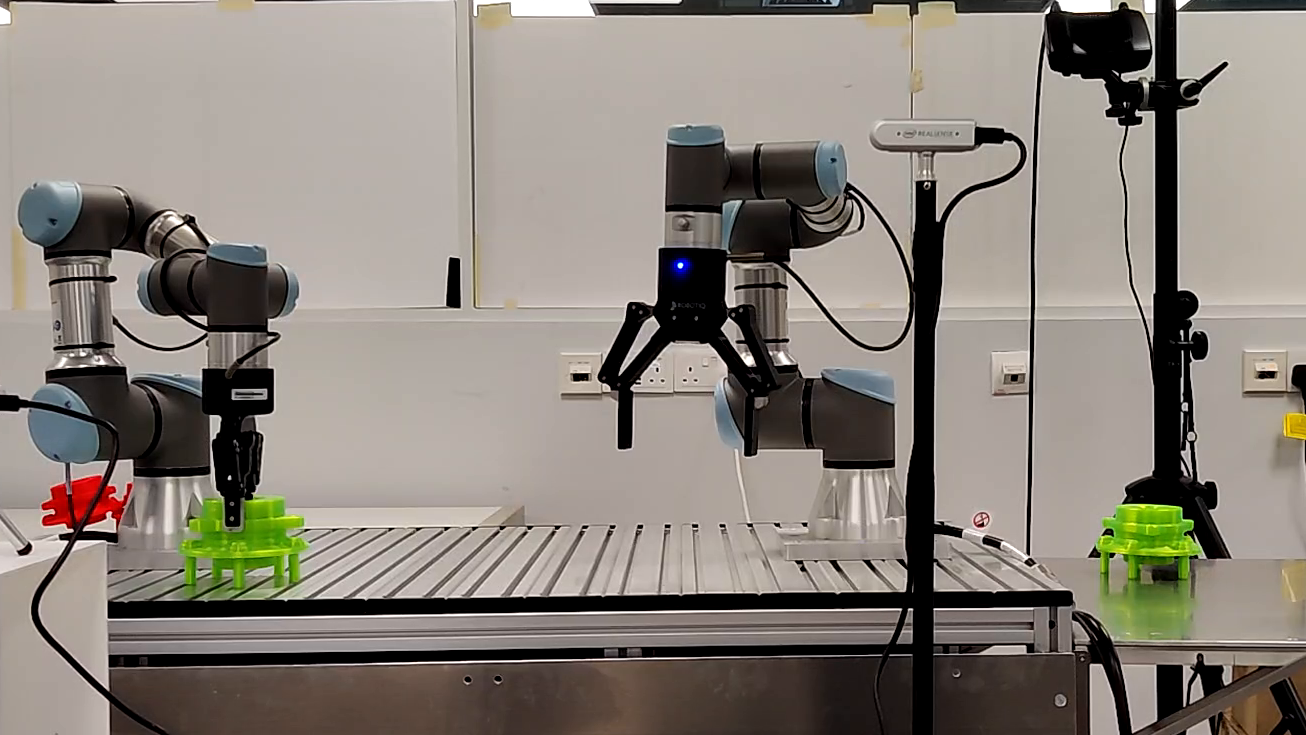}
			\caption{Discard} \label{fig:discard}
		\end{subfigure}%
		
		\caption{Snapshots of two UR3e robots realizing the logic action sequence found by our tree search for the inspection of two bearings. The video of the experiment is available at https://youtu.be/WiJU70f3EFw.} 
		\label{fig:key_trafo}
	\end{figure}
	
	Task planning in robotics has a long-running history with early beginnings in~\cite{strips1971}. The search for a symbolic goal state can be conducted for example with MCTS~\cite{Ren2021} which was first drafted in~\cite{Coulom206}. The approach in~\cite{Kocsis2006} constitutes an important extension by applying the node selection criterion UCB (Upper Confidence Bound) to trees (UCT) to guide the tree search. It balances exploitation, i.e. following promising search directions with high average reward resulting from chosen actions, and exploration, i.e. following unexplored search directions. Another important addition is Rapid Action Value Estimation (RAVE) which shares knowledge between related nodes about their actions and resulting rewards and enhances the performance of the tree search~\cite{gelly2011}. The combinatorial complexity of the search can be softened for example as proposed in~\cite{Lertkultanon2015} by using a grasp-placement table which delegates part of the combinatorial tree search into an offline process and looks up possible action sequences online. The work in~\cite{Driess2021} is based on deep learning that can deduct logic action sequences and motion plans from a scene image. The tree search is conducted only offline during the training process.
	
	The previously mentioned work in~\cite{strips1971} uses the concept of sub-goals to break down the overall goal by identifying the difference to the current state. 
	On the other hand, the work in~\cite{Davidov2006} explicitly considers the scenario of efficiently reaching pre-defined sub-goals. They implement a distance-based heuristic in order to proceed the search into a direction where the most sub-goals can be achieved while minimizing the associated cost. The authors suggest a machine-learning based method to estimate the heuristic online.
	The approach in~\cite{Henrich1998} uses the A\hspace{-1pt}\mbox{*} algorithm with an adapted heuristic to approach the physical goal configuration that is closest to the current configuration. Similarly, adaptive A\hspace{-1pt}\mbox{*} is used in~\cite{Matsuta2010} to identify the best path to traverse multiple goal locations. The algorithm proposed in~\cite{Ishida2019} functions similarly but is augmented with an additional pruning method that removes vertices that lead to goals that already have been visited at a minimum cost. The algorithm thereby only considers a subset of the overall goal state in order to further reduce the computational burden.

	In this paper, we propose a fast depth-first search algorithm for multi-goal task planning based on MCTS with PNE that prioritizes promising nodes and expands them depth first. Unlike Hierarchical MCTS~\cite{Vien2015} the prioritization is applied within the tree itself rather than identifying hierarchical structures in the action space. 
	In contrast to~\cite{Waard2016} no pre-definition of action sequences is necessary as such accurate domain knowledge may often not be available. Furthermore, the prioritization is strict (only a single branch is expanded) and does not require tuning as compared to Beam Monte Carlo Tree Search~\cite{Baier2012}. Here at each depth of the tree, the most promising nodes are selected according to some heuristic (for example node visits with a threshold) while the others are pruned.
	One critical aspect of depth-first searches is to avoid infinite branch expansions. We realize this by introducing a so-called `bridging' factor that limits the prioritization of nodes to a certain amount of iterations if no sub-goal is encountered. Together with action reduction (AR) based on a kinematic feasibility criterion, our algorithm proves to be very efficient in planning long action sequences for problems with a large number of sub-goals. The algorithm is designed to identify a first (and possibly sub-optimal) symbolic task plan as quickly as possible. 
	
	First, we describe the problem at hand, see Sec.~\ref{sec:problem}.
	We extend MCTS with PNE which we explain in Sec.~\ref{sec:algo}. We detail the algorithm's operational complexity and show that it is linear in the number of sub-goals (Sec.~\ref{sec:complexity}). We also shortly address the implications for our algorithm if it is applied to stochastic decision processes.
	The AR and the kinematic reachability criterion are outlined in Sec.~\ref{sec:ar}. Finally, the algorithm is evaluated on a task planning example for industrial robot inspection and demonstrated on real robots, see Sec.~\ref{sec:eval}.

	\section{Problem description}
	\label{sec:problem}
	
	In this work, we consider task planning scenarios where one or more symbolic goal states have to be achieved. Specifically, we consider the industrial use case of inspecting a number of wheel bearings manipulated by two stationary industrial robot arms equipped with one gripper each. The inspection pipeline for $b$ bearings is set up in the following order:
	\begin{enumerate}
		\item Inspect all sides (si) of the bearing (be) by presenting (pr: presented) it to one of the two inspection cameras (cam).
		\item Once inspected, place the bearing on a table for human inspection (plhi: placed for human inspection).
		\item Discard (di: discarded) the inspected bearing on a conveyor belt.
	\end{enumerate}
	We translate this problem into computer language by means of the Planning Domain Definition Language~\cite{pddl1998}. 
	The desired multi-goal state $\mathcal{G}_{d}$ then becomes 	
	\begin{align}
		\mathcal{G}_{d} = \{& \text{\{pr,be$_1$,si$_1$\}},\dots,\text{\{plhi,be$_1$\}},\text{\{di,be$_1$\}},\dots,  \nonumber \\
		& \text{\{pr,be$_b$,si$_1$\}},\dots,\text{\{plhi,be$_b$\}},\text{\{di,be$_b$\}}		
		\}\label{eq:goal}
	\end{align}
	It encompasses the desired overall goal of `inspect, place for human inspection and discard all $b$ bearings'.
	$\mathcal{G}_{d}$ consists of $N\coloneqq\text{card}(\mathcal{G}_{d}) = 6b$ ($\text{card}$: cardinality) \textit{sub-goals}, for example, `side 1 of bearing 1 is presented' (`pr,be$_1$,si$_1$) or `bearing $b$ is discarded' (`di,be$_b$').
	The logic language includes a set of \textit{objects} (grippers, tables, bearings, cameras).
	Furthermore, a set of \textit{predicates} indicates the current state of objects. Examples would be `free(gr)' (the gripper is currently not holding a bearing) and `placed(be$_i$,tab$_j$)' (bearing $i$ is placed on table $j$). The corresponding \textit{actions} that can change these predicates are `pick(gr$_1$,be$_i$)' and `place(gr$_1$,be$_i$,tab$_j$)'. We refer to a sequence of actions as \textit{path}.

	\begin{remark}
		\label{remark:abb}
		Our algorithm is applicable to any scenario with a symbolic sub-goal structure as described in~\eqref{eq:goal}; the specific naming of objects, actions, predicates can be subject to variation. 
		Furthermore, the above overall goal $\mathcal{G}_{d}$ could be abbreviated as `bearing $b$ is discarded' and associated with the predicates `if bearing $i$ has been presented', `if bearing $i$ has been placed for human inspection' (with $i = 1,\dots,b$) and `if bearing $j$ has been discarded' (with $j = 1,\dots,b-1$). In this case, our algorithm would apply the classical MCTS algorithm and ignore the special structure of the goal state. This is less efficient as is demonstrated in Sec.~\ref{sec:eval}. On the other hand, our proposed algorithm has the same complexity as solving each sub-goal as its own task planning problem, see Sec.~\ref{sec:complexity}. However, this approach would require a complicated implementation that transfers initial states between the problems and solves the individual problems in the correct order such that required predicates are met. At the same time, only a sub-optimal solution can be obtained as rewards over the whole problem are ignored (unless they are somehow transferred between the sub-goal task planning problems).
	\end{remark}

	\section{A priority-based tree search for multi-goal task planning}
	\label{sec:logic}

	\algtext*{EndIf}

	In this work, we employ Monte-Carlo tree search based on the UCT criterion~\cite{Kocsis2006}
	\begin{align}
		\text{UCT}(i) = \frac{w_i}{n_i \sigma_i^{\kappa}} + c\sqrt{\frac{\ln{\mathcal{N}_i}}{n_i\sigma_i^{\kappa}}}
	\end{align}
	The index $i$ indicates the current node in the tree. $w\geq 0$ is the reward that indicates the sum of the rewards of its child nodes. In this work a node is rewarded if its correspondent symbolic state corresponds to a sub-goal. $n \geq 1$ is the number of node visits and $\mathcal{N} \geq 1$ is the number of visits of the node's parent. $c>0$ is an empirically chosen bias parameter to balance exploration (second term) with exploitation (first term).
	We add an additional term $\sigma\geq 1$ (with the tuned parameter $\kappa=3$) which is the sum of the occurrences of the action at the current node's parents. This penalizes re-occurrences of the same action in order to favor diverse action sequences (if no further randomization is implemented, actions defined first in the domain file will appear more frequently than other actions with the same UCT value but defined later). Due to the higher associated computational load, we add this term only for prioritized nodes as described below.
	
	One disadvantage of the UCT criterion is that pure exploitation ($c=0$) leads to algorithm failure since the tree is expanded to infinite depth along the first node encountering a reward. If the focus on exploitation is desired careful tuning of the rewards $w>0$ and bias parameter $c>0$ would be required. Here, we  propose a greedy tree search method based on PNE which achieves pure exploitation without being trapped in infinite depth-first node expansions. It originates from the observation that nodes purely chosen according to the UCT criterion lead to the tree search getting `distracted' by other promising nodes along the first encountered promising path. For example, in our use-case of bearing inspection it is irrelevant whether the sides, top, or bottom of the bearings are inspected first. It is more beneficial to expand one of the possibilities depth first until the terminal state is reached rather than exploring several promising options in a parallel fashion.
	
	\subsection{Description of the algorithm}
	\label{sec:algo}	
	
	\makeatletter
	\newcommand\fs@betterruled{%
		\def\@fs@cfont{\bfseries}\let\@fs@capt\floatc@ruled
		\def\@fs@pre{\vspace*{5pt}\hrule height.8pt depth0pt \kern2pt}%
		\def\@fs@post{\kern2pt\hrule\relax}%
		\def\@fs@mid{\kern2pt\hrule\kern2pt}%
		\let\@fs@iftopcapt\iftrue}
	\floatstyle{betterruled}
	\restylefloat{algorithm}
	\makeatother
	\begin{algorithm}[h]
		\caption{MCTS with PNE}\label{alg:mctspne}
		\begin{algorithmic}[1]
			\Statex\textbf{Input:} {$\mathcal{G}_{d}$, $n_0$, $\beta$}
			\Statex \textbf{Output:} Goal path
			\State $\mathcal{G} = \{\}$, $p=1$, $\mathcal{T}_1 = \{n_0\}$, $\mathcal{T} = \{\mathcal{T}_1\}$ 
			\While{$\mathcal{G} \neq \mathcal{G}_{d}$}
			\State $\hat{n} = {\tt chooseNode}(\mathcal{T}_p)$
			\State $\mathcal{T}_p \leftarrow \mathcal{T}_p \setminus \hat{n}$
			\State $\hat{p} = p$
			\State  $\mathcal{A}(\hat{n}) \coloneqq \{a_1,\dots,a_k\} = {\tt find\mathcal{A}}(\hat{n})$ 
			\For{$a$ in $\mathcal{A}(\hat{n})$} \label{line:for}
			\State $n, s_n, w = {\tt simulateAction}(\hat{n}, a)$
			\State ${\tt backpropagate}(n,w)$~\label{line:backprop}
			\State $\beta_n = \beta_{\hat{n}}$
			\If{$\beta=0$ or ($s_n \notin \mathcal{G}$ and $s_n \in \mathcal{G}_{d}$)}
			\State $\beta_n = 0$
			\State $\mathcal{G}\leftarrow \mathcal{G}\cup s_n\texttt{\texttt{}}$ 
			\If{$\beta > 0$ and $p = \hat{p}$ and $\mathcal{T}_p \neq \emptyset$} 
			\State $p\leftarrow p+1$
			\State $\mathcal{T} = \{\mathcal{T}_l\vert l =1,\dots,p\}$ with $\mathcal{T}_p = \{\}$
			\EndIf
			\State $\mathcal{T}_p \leftarrow \mathcal{T}_p\cup n$ 
			\Else
			\State $\beta_n \leftarrow \beta_n+1$
			\If{$\beta_n = \beta$} 
			\State 	$\beta_n =0$
			\State $\mathcal{T}_{p-1} \leftarrow \mathcal{T}_{p-1}\cup n$
			\IIf{$\mathcal{T}_p = \emptyset$} $p \leftarrow p-1$
			\EndIf
			\EndIf
			\EndFor
			\EndWhile
			\State \Return ${\tt getPath}(\mathcal{G})$
		\end{algorithmic}
	\end{algorithm}	
	
	The algorithmic overview of our prioritized tree search is given in Alg.~\ref{alg:mctspne}. Our algorithm corresponds to the original MCTS algorithm~\cite{Kocsis2006} for $\beta=0$, also see Sec.~\ref{sec:beta}. Our algorithm differs by virtue of strictly prioritizing nodes that already have lead to sub-goals. We achieve this by subdividing the tree $\mathcal{T}$ into node sets of $p$ priority levels $\mathcal{T} = \{\mathcal{T}_l\vert l=1,\dots,p\}$. 
	
	The tree is initialized with a single level $p=1$. The root node $n_0$ (with some associated initial state $s_{n_0}$) of the otherwise empty tree is then contained in the node set $\mathcal{T}_1= \{n_0\}$ of level $l=1$. Initially we have $\mathcal{T} = \{\mathcal{T}_1\}$.
	
	First, the routine ${\tt chooseNode}$ chooses the node $\hat{n}$ with the highest UCT value of the node set of the highest priority level $\mathcal{T}_p$.
	The node $\hat{n}$ is consequently removed from $\mathcal{T}_p$. The possible action space $\mathcal{A}(\hat{n})$ at node $\hat{n}$ is determined by ${\tt find\mathcal{A}}$. Each action $a$ is simulated by ${\tt simulateAction}$, resulting in a new node $n$ with associated state $s_n$ and a reward $w$  (in our deterministic case this step simply assigns a uniform reward if the action leads to a sub-goal). The reward $w$ is back-propagated to the node's parents by ${\tt backpropagate}$. The algorithm checks whether the state $s_n$ corresponds to a sub-goal ($s_n\in\mathcal{G}_{d}$) that has not been encountered yet ($s_n\notin\mathcal{G}$). If this is the case the current goal state  is augmented to $\mathcal{G} = \{\mathcal{G},s_n\}$. Finally, the node $n$ is added to the new node set $\mathcal{T}_p$ of priority level $p\leftarrow p+1$. For higher algorithmic efficiency we could exit the for-loop (line~\ref{line:for}) as soon as a sub-goal is encountered. For better readability we do not further address this option.
	
	The increment $p\leftarrow p+1$ is only done once per expansion iteration. Furthermore, $p$ is only increased if this does not lead to an empty set $\mathcal{T}_p = \emptyset$ of the current highest priority level $p$. This is due to our usage of the so-called bridging factor $\beta$ which retains a node's expanded children $n\in\mathcal{N}$ on the same priority level $\beta$ times even if no sub-goal is encountered ($s_n\notin \mathcal{G}_{d}$). With an empty set $\mathcal{T}_p = \emptyset$, a node would be retained twice, once on the new level $p+1$, and once on the old level $p$, skewing the intended bridging factor. Each node has thereby an associated counter $\beta_n$ which indicates how many times it has been retained on the same priority level. If no sub-goal is encountered, $\beta_n$ is incremented. If this happens $\beta$ times the node's level is decreased and the counter is reset to zero. Similarly, if a sub-goal is encountered the node is elevated to the next level and the counter is reset as well.
	
	The algorithm terminates if all sub-goals have been encountered such that $\mathcal{G} = \mathcal{G}_{d}$. The corresponding path extracted by ${\tt getPath}$ is returned. The algorithm is designed to find a first (possibly sub-optimal with respect to the path length or accumulated rewards along the path) solution as quickly as possible by expanding a minimum number of nodes. The optimal solution can be identified by increasing the run-time of the algorithm without terminating it at the first solution.
	
	\subsection{Bridging factor}
	\label{sec:beta}
	
	The usage of the bridging factor $\beta$ forces the expansion of a sub-tree starting from the current prioritized node as its root until a sub-goal is encountered. By limiting the depth of the sub-tree to $\beta$, infinite tree expansions are prevented.
	
	If the bridging factor is chosen as $\beta=1$ the priority of a node is reduced by one to $l\leftarrow l-1$ if it does not result in a sub-goal. Any value larger than one leads to retention of $\beta$ expansion iterations on the same priority level even if the node does not correspond to a sub-goal. If the bridging factor is chosen as $\beta=0$, the expanded nodes are not subject to prioritization. The algorithm corresponds to the original MCTS tree search algorithm. 
	
	The bridging factor $\beta$ can be tuned according to specific domain knowledge. In our case of bearing inspection, an occluded face of the bearing can be inspected by re-grasping the bearing. This can be achieved by placing the currently held bearing and by picking it up again with a different grasp. The bridging factor could therefore be set accordingly to $\beta=3$ (place - pick - present). 
	From our experience, it is thereby beneficial to rather choose a too large than a too small bridging factor. Specifically, if the bridging factor is chosen large enough (or \textit{correctly}), for deterministic decision processes the search complexity  becomes linear in the number of sub-goals as we show in the next section~\ref{sec:complexity}.

	\subsection{Search complexity}
	\label{sec:complexity}
	
	In the following, we detail the complexity of our search algorithm for deterministic decision processes and show that it is linear in the number of sub-goals if the bridging factor is chosen correctly.
	In the following, $\mathbb{A} = \text{card}(\mathcal{A}_{\max})$ is the maximum number of possible actions that can occur at any node. 
	
	\begin{definition}[Correct bridging factor] 
		We refer to the bridging factor $\beta$ as being \textit{correct} if at any stage of the tree search any sub-goal can be achieved within a sub-tree of depth~$\beta$.
	\end{definition}
	
	\begin{theorem}
		If the bridging factor is chosen correctly the algorithm converges to the terminal goal state $\mathcal{G}_{d}$ within a tree of at most $N \mathbb{A}^{\beta}$ nodes. 
	\end{theorem}
	
	\begin{proof}
		Let the search start with a root node in an otherwise empty tree. Due to the assumption of a correctly chosen bridging factor, we find a sub-goal that augments the current state to $\mathcal{S} = \{s_i\}$ with $i \in \{1,\dots,N\}$ within a sub-tree of at most $\mathbb{A}^{\beta}$ nodes. This node is then elevated to the next priority level and, due to our hierarchical algorithm structure, picked for the next round of expansions as the root node of a new sub-tree. Again, due to our assumption of a correct bridging factor, we find a sub-goal that augments the current state to $\mathcal{S} = \{s_i, s_j\}$ with $j \in \{1,\dots,N\} \setminus i$ within a sub-tree of at most $\mathbb{A}^{\beta}$ nodes. The tree is currently expanded to $2\mathbb{A}^{\beta}$ nodes. The algorithm repeats this procedure overall $N$ times until the goal state $\mathcal{S}_g$ is reached. The search complexity is therefore upper bounded by $N \mathbb{A}^{\beta}$.
	\end{proof}
	
	This is in contrast to a breadth-first approach with a single tree of depth $N\beta$ and at most $\mathbb{A}^{N\beta}$ nodes (since with the assumption of a correct bridging factor the goal state is necessarily contained within). On the other hand, our algorithm has the same complexity as solving each sub-goal as its own task planning problem. As remarked in Rem.~\ref{remark:abb}, this requires more complicated implementations. If the bridging factor is chosen as $\beta = 0$ the algorithm has the same complexity as the original MCTS algorithm.
	
	For stochastic decision processes the definition of correctness of the bridging factor is not applicable. Rather, we only have correctness of the probabilistic bridging factor $\beta_{\mathcal{P}}$ \textit{to a probability} ${\mathcal{P}}$. In our example of bearing inspection, assume that the inspection camera has the probability $f$ of failing to trigger the shutter when the bearing is presented (this is reflected in the simulation step ${\tt simulateAction}$ of the action `present'). Furthermore, assume that $\beta=3$ (place - pick - present) is the correct bridging factor for inspecting a side of the bearing in the deterministic case. The  bridging factor $\beta_{\mathcal{P}} = 2+u$ is then correct to probability ${\mathcal{P}} = 1-f^u$ that a side of the bearing is inspected successfully (by virtue of repeating the present action $u > 0$ times). Similarly, this probability also holds for the search complexity.
	
	\subsection{Example visualizing PNE}
	
	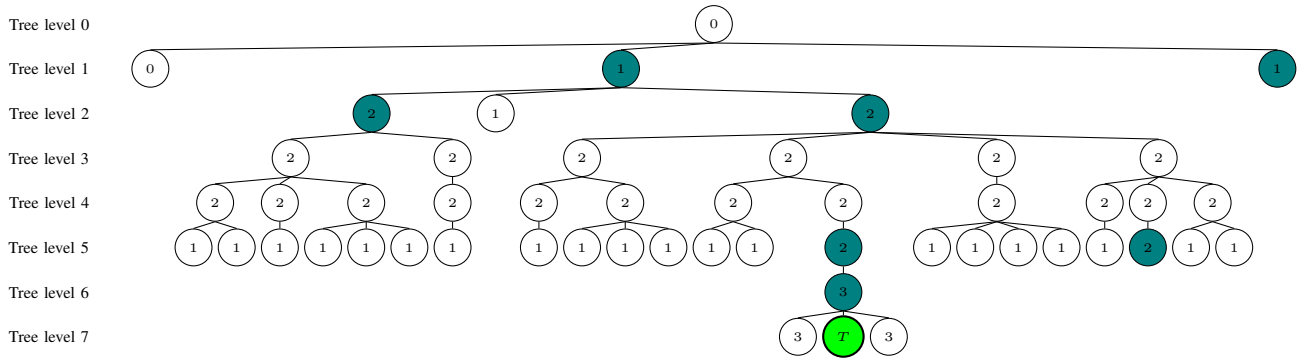
\begin{figure*}
		\hspace{5pt}
		\centering
		\resizebox{2.0\columnwidth}{!}{%
			\begin{tikzpicture}[thick,scale=0.8, every node/.style={scale=0.8}
				,Comment/.style = {
					shape=rectangle
					,  draw=none
					,  inner sep=0mm
					,  outer sep=0mm
					,  minimum height=5mm
					,  align=right
				}]
				\tikzset{font=\tiny,
					level distance=0.6cm,
					every node/.style=
					{
						circle,draw,
						align=center
					}
					,
					edge from parent/.style=
					{draw}
				}
				\Tree 
				[.\node (N0) {$0$};
				[.\node(N1){$0$}; ]
				[.\node[fill=teal]{$1$};
				[.\node[fill=teal](N2){$2$}; 
				[.\node(N3){$2$}; 
				[.\node(N4){$2$}; 
				[.\node(N5){$1$}; ]
				[.{$1$} ]]
				[.{$2$} 
				[.{$1$} ]]
				[.{$2$} 
				[.{$1$} ]
				[.{$1$} ]
				[.{$1$} ]] ]
				[.{$2$} 
				[.{$2$} 
				[.{$1$} ]]]]
				[.{$1$} ]
				[.\node[fill=teal]{$2$};
				[.{$2$} 
				[.{$2$} 
				[.{$1$} ]]
				[.{$2$} 
				[.{$1$} ]
				[.{$1$} ]
				[.{$1$} ]]]
				[.{$2$} 
				[.{$2$} 
				[.{$1$} ]
				[.{$1$} ]]
				[.{$2$} 
				[.\node[fill=teal]{$2$};
				[.\node[fill=teal](N6){$3$};
				[.{$3$} ]
				[.\node[fill=green,thick](N7){$T$};]
				[.{$3$} ]]
				]]]
				[.\node[]{$2$}; 
				[.\node[]{$2$}; 
				[.\node[]{$1$}; ]
				[.\node[]{$1$}; ]
				[.\node[]{$1$}; ]
				[.\node[]{$1$}; ]]]
				[.\node[]{$2$}; 
				[.\node[]{$2$}; 
				[.\node[]{$1$}; ]]
				[.\node[]{$2$}; 
				[.\node[fill=teal]{$2$}; ]]
				[.\node[]{$2$}; 
				[.\node[]{$1$}; ]
				[.\node[]{$1$}; ]]]]]
				[.\node[fill=teal]{$1$}; ]
				]
				\node[Comment] (L0)    [left=65.0mm of N0] {Tree level 0};
				\node[Comment] (L1)    [left=4.5mm of N1] {Tree level 1};
				\node[Comment] (L2)    [left=28.25mm of N2] {Tree level 2};
				\node[Comment] (L3)    [left=19.6mm of N3] {Tree level 3};
				\node[Comment] (L4)    [left=11.45mm of N4] {Tree level 4};
				\node[Comment] (L5)    [left=9.2mm of N5] {Tree level 5};
				\node[Comment] (L6)    [left=78.95mm of N6] {Tree level 6};
				\node[Comment] (L7)    [left=78.7mm of N7] {Tree level 7};
			\end{tikzpicture}
		}
		\caption{Graphical overview of PNE with the incorrectly chosen bridging factor $\beta=3$ and $N=5$ sub-goals. The node number indicates the current priority level of the node. Dark green nodes are sub-goal states. Light green nodes are terminal nodes where $\mathcal{G} = \mathcal{G}_{d}$.}
		\label{fig:PNE}
	\end{figure*}
	
	An example with a graphical overview of PNE is given in Fig.~\ref{fig:PNE}. The number of sub-goals is $N=5$ (\{pr,be,si$_1$\},\{pr,be,si$_2$\},\{pr,be,si$_3$\},\{plhi,be\},\{di,be\}). The bridging factor is incorrectly chosen as $\beta=3$ (for the deterministic case; for the stochastic case $\beta_p$ is correctly chosen to a probability $p<1$).
	
	On the first tree level, two sub-goal nodes (\{pr,be,si$_1$\}, \{pr,be,si$_2$\}) are encountered and elevated to the node set $\mathcal{T}_1$ of the first priority level $p=l=1$. The prioritized node with the highest UCT value is expanded (\{pr,be,si$_1$\}). Two sub-goal nodes (\{pr,be,si$_2$\},\{pr,be,si$_3$\}) are encountered and elevated to the second priority level $\mathcal{T}_2$. The leftmost branch (\{pr,be,si$_2$\}) is expanded to a sub-tree of depth $3$ without encountering a sub-goal. Due to the bridging factor $\beta=3$, the sub-tree search is aborted and the priority levels of all child nodes on tree level 5 are reduced to $l=1$. Potentially, once the priority level is reduced to 0 the node leaves the set of prioritized nodes entirely and is subject to the normal UCT-based selection, given that the sets of prioritized nodes is empty ($\mathcal{T} = \emptyset$). 
	
	The second node on priority level 2 on tree level 2 \{pr,be,si$_3$\} is expanded. Two sub-goals (\{pr,be,si$_2$\},\{plhi,be\}) occur at sub-tree level 3 (tree level 5), for example after the bearing has been placed and re-grasped. The corresponding nodes are kept on priority level $2$ since otherwise this priority level would be left empty. This prevents a skewed bridging factor as the node and its children would be retained twice, first on priority level $l+1=3$ and then on $l=2$. The first expanded node leads to the goal state $\mathcal{S}_g$ at tree level 7 (after discovering the sub-goals \{plhi,be\} and \{di,be\}). Note how the second sub-goal node on tree level 1 ( \{pr,be,si$_2$\}) is never expanded, independent of the reward $w>0$. This would not necessarily be the case for a tree search purely based on UCT, depending on the choice of reward $w$ and bias $c$.
	
	\subsection{Action reduction (AR)}
	\label{sec:ar}
	
	In order to soften the combinatorial complexity of the tree search, we define `global' actions which are removed from the action space if they turn out to be `globally' infeasible. 
	If a global action first occurs in $find\mathcal{A}$ during the tree search, it is checked for its kinematic feasibility by solving the associated inverse kinematics problem. In case this action is infeasible (i.e. the inverse kinematics problem can only be solved with an error above a certain threshold), we remove it from the action space and mark it as globally infeasible. This prevents future occurrences during the search across the whole tree (and not only on a single branch) by sharing knowledge between nodes similarly to RAVE~\cite{gelly2011}.
	
	Such global actions have to be well defined and should not become feasible in later stages of an action sequence. In our use case, this would be a fixed-base robot that meets invariant kinematic reachability conditions of its workspace. For example, a fixed table out of reach can never be reached since the robot is immobile. This eliminates any actions that are associated with the robot and the table.

	\section{Evaluation}
	\label{sec:eval}

	\begin{table*}[htp!]
		\centering
		\resizebox{2.\columnwidth}{!}{%
			\begin{tabular}{@{} ccccccccccccccc @{}}  
				\\
				\toprule
				& & UCT & \multicolumn{4}{c}{UCT+AR} & \multicolumn{8}{c}{UCT+AR+PNE} \\
				\cmidrule(lr){3-3} \cmidrule(lr){4-7} \cmidrule(lr){8-15}
				& & & & & & & \multicolumn{4}{c}{$\beta=2$} & \multicolumn{4}{c}{$\beta=5$} \\
				\cmidrule(lr){8-11} \cmidrule(lr){12-15}
				$b$ & $N$ & $\mathbb{D}$ / $\mathbb{L}$ & $\mathbb{D}$    & $\mathbb{L}$ & $\mathbb{I}$ / $\mathbb{A}$ & $\mathbb{T}_L$ / $\mathbb{T}_A $ & $\mathbb{D}$ & $\mathbb{L}$ & $\mathbb{I}$ / $\mathbb{A}$  & $\mathbb{T}_L$ / $\mathbb{T}_A $ &  $\mathbb{D}$ & $\mathbb{L}$ & $\mathbb{I}$ / $\mathbb{A}$ & $\mathbb{T}_L$ / $\mathbb{T}_A $\\
				\cmidrule(lr){1-1} \cmidrule(lr){2-2} \cmidrule(lr){3-3} \cmidrule(lr){4-7} \cmidrule(lr){8-11} \cmidrule(lr){12-15}
				1 & 6 & - & 231 & 18 & 7 / 16 & 0.1 / 5.0 & 651 & 25 & 7 / 16 & 0.1 / 5.1 & 127 & 24 & 7 / 16 & 0.1 / 5.3\\
				2 & 12 & - & -  & - & 7 / 16 & 3.7 / 4.1 & 601 & 35&  7 / 18 & 0.2 / 5.1 & 139 & 35 & 7 / 17 & 0.1 / 4.7\\
				3 & 18 & - & - & - & 10 / 25 & 3.8 / 6.6& 1329 & 52 & 10 / 26 & 0.4 / 7.7 &301 & 52 & 10 / 24 & 0.1 / 6.4 \\
				4 & 24 & - & - & - & 13 / 36 & 4.0 / 10.7 & 2747& 69 & 13 / 35 & 1.2 / 10.8 & 652 & 69 & 13 / 32 & 0.3 / 9.1 \\
				5 & 30 & - & - & - & 12 / 44 & 4.0 / 11.6 & 6101 & 86 & 16 / 45  & 5.4 / 14.7 & 1371 & 86 & 16 / 41 & 0.6 / 12.0\\
				6 & 36 & - & - & - & 14 / 58 & 4.2 / 16.7 & 18968 & 103 & 19 / 56 & 61.7 / 19.1 &2709 & 103 & 19 / 51 & 1.3 / 15.7\\
				7 & 42 & - & - & - & 16 / 63 & 4.4 / 16.5 & 23700 & 120 & 21 / 64  & 98.4 / 24.0 &3759 & 128 & 21 / 57 & 2.3 / 18.7\\
				8 & 48 & - & - & - & 18 / 79 & 4.6 / 21.4 & - & - & 24 / 82 & 145.4 / 33.0 & 6413 & 145 & 24 / 68  & 4.3 / 22.9 \\
				9 & 54 & - & - & - & 20 / 97 & 4.7 / 25.3 & -&- & 27 / 91  & 72.8 / 35.7& - & - & 27 / 79 & 51.0 / 26.9\\
				\bottomrule
			\end{tabular}
		}
		\caption{Tree search on the bearing inspection problem. $\mathbb{D}$ is the number of expanded tree nodes at convergence $\mathcal{G}=\mathcal{G}_d$. $\mathbb{L}$ is the length of the identified action sequence. $\mathbb{A}$ is the number of encountered actions and $\mathbb{I}$ is the number of infeasible ones. $\mathbb{T}_L$ and $\mathbb{T}_A$ are the computation times in seconds for the tree search and inverse kinematic checks, respectively. Empty entries `-' indicate that no solution was found within the limit of $\mathbb{D}=3 \cdot 10^4$ tree nodes.}
		\label{tab:treeEval}
	\end{table*}
	
	We evaluate our developments in the industrial scenario of bearing inspection described in Sec.~\ref{sec:problem}.
	The bearings can only be placed on a table reachable by one of the two robots. Each table has $b$ different spots and each spot can hold only a single of the $b$ bearings. Each of the two inspection cameras can only be reached by one of the robots. A scene image is given in Fig.~\ref{fig:key_trafo}. 
	The overall action space consists of 17 different actions with additional parametrization in the used robots, bearings, and table spots.
	Both robots are equipped with grippers of different sizes. The smaller gripper can grasp the bearing only from the top by inserting a finger into the bearing shaft. The bigger gripper can fully grasp the bearing from the top as well as from the side. The bearing can be passed between the robots via handover. Four different present actions for the full inspection of the bearing without occlusion are considered: one each for the top and the bottom and two for the sides of the bearing.
	
	The overall inspection pipeline consists of three stages:
	\begin{enumerate}
		\item Object localization.
		\item Logic sequence by MCTS with PNE and AR.
		\item Robot trajectory planning by OpenRAVE \cite{diankov2010automated}.
	\end{enumerate}
	The detected bearing locations are used to identify kinematically feasible logic action sequences to achieve the desired symbolic goal state. The inverse kinematics problems for AR are solved to low accuracy by a fast non-linear local optimizer based on the penalty method. In the final stage, a high-accuracy robot joint trajectory including obstacle avoidance is determined by OpenRAVE. We make the simplifying assumption that the logic action sequence is always realizable by the robots. This can be justified by the limited complexity of the scenario given by the fixed-base manipulators.
	
	In the following, we evaluate the tree search algorithm (Sec.~\ref{eval:treesearch}) and apply the pipeline to the real-robot inspection scenario (Sec.~\ref{eval:realrob}).
	
	\subsection{Tree search}
	\label{eval:treesearch}

	In this section, we consider the tree search for our bearing inspection problem. We test our algorithm as the combination of UCT, AR, and PNE (UCT+AR+PNE) against UCT~\cite{Kocsis2006} and UCT+AR for the inspection of up to 9 bearings. Thereby, we evaluate PNE with the two different bridging factors $\beta=2$ and $\beta=5$. We set $c=\sqrt{2}$ and assign the uniform reward $w=1$ to nodes whose states correspond to sub-goals (presented, placed for human inspection, discarded). The other states are unrewarded (picked, placed, ...). 
	
	The results are summarized in table~\ref{tab:treeEval}.
	Our proposed method UCT+AR+PNE ($\beta=5$) is the only method that finds solutions $\mathcal{G}=\mathcal{G}_d$ for up to $b=8$ bearings ($N=48$) within the tree size limit of $3\cdot 10^4$ nodes (6413 nodes at convergence). UCT alone does not find solutions for any number of bearings. UCT+AR only finds a solution for one bearing.
	
	For UCT+AR+PNE ($\beta=5$) the dimension of the tree ($\mathbb{D}$) grows approximately linearly in the number of bearings (for 1 bearing $\mathbb{D}=127$, for 8 bearings $\mathbb{D}=6413\approx 7\cdot8 \cdot 127$ with a small factor of 7). This confirms our previous statement in Sec.~\ref{sec:complexity} regarding the linear search complexity in the number of sub-goals $N$ for a correctly chosen bridging factor. In the same vein, it can be observed that $\beta=2$ is not a correctly chosen bridging factor due to the higher number of expanded nodes due to the tendency of a MCTS inspired exploratory search (since no sub-goal can be found within the sub-trees of depth 2).
	For all number of bearings, the actual search space is significantly less than the theoretical upper bound (for $b=8$: 6413 nodes compared to $48\cdot 44^5$ nodes with $N=48$, $a=68-24 =44$, $\beta = 5$). This is due to the fact that the actual action space at each node is significantly smaller than the possible maximum one due to given predicates. Furthermore, a sub-goal is oftentimes encountered at a much earlier stage of the sub-tree than at the conservatively chosen depth~$\beta$.
	Due to the limited number of expanded nodes, the computation times for the tree search $\mathbb{T}_L$ are under 1~s for up to 5 bearings for UCT+AR+PNE ($\beta=5)$. This is in contrast to UCT+AR where the tree search takes longer than 1~s if a logic action sequence for more than 1 bearing is to be found. 
	However, if the number of expanded nodes increases due to an incorrectly chosen bridging factor ($\beta=2$) the tree search becomes expensive due to the more expensive UCT criterion relying on the factor $\sigma$ as the sum of occurrences of actions along its path (145.4~s for $3\cdot10^4$ expanded nodes, 8 bearings).
	At the same time, solving the inverse kinematics problems for AR takes $\mathbb{T}_A=12$~s for 5 bearings for UCT+AR+PNE ($\beta=5$, time for tree search $\mathbb{T}_L=0.6$~s). Overall 41 inverse kinematics problems are solved of which 16 turn out to be infeasible.
	
	The solutions of UCT+AR+PNE for one bearing are more expensive with a longer path length ($\mathbb{L}$) of 25 compared to 18 for UCT+AR. This is due to the fact that UCT+AR+PNE tends to find sub-optimal solutions with a larger gap between encountered sub-goals. This can also be observed for the case of 7 bearings with a path length of 120 and 128 for UCT+AR+PNE with $\beta=2$ and $\beta=5$, respectively. If the bridging factor is correctly chosen the maximum resulting path length $\mathbb{L}$ is given by $N\beta$. It can be observed that this upper bound holds only for the correctly chosen bridging factor $\beta=5$ for all number of bearings where a solution is found.

	\subsection{Experiment on real robot}
	\label{eval:realrob}
	
	We realize the logic action sequence from Sec.~\ref{eval:treesearch} for two bearings on two UR3e robots identified by UCT+AR+PNE with the bridging factor $\beta=5$ (sequence of 35 actions). Due to the complexity of the actions (for example handovers) we hard-code all of their parameters like gripper locations and orientations. The only variability is given by the object localization and collision avoidance. This is in contrast to task and motion planning (TAMP) approaches like LGP~\cite{Toussaint2015} where a constrained optimization problem over the whole kinematic trajectory is solved. 
	
	Initially, one bearing is placed on the left table and the other one on spot 1 of the right table. The bearing on the left is picked up first, partly presented to the left inspection camera, and handed over to the robot on the right. The handover includes force control as described in~\cite{Pfeiffer2017}. After the bearing has been placed on spot 2 of the right table we run the object localization to account for inaccuracies in the handover procedure. The remaining sides of the two bearings are then sequentially picked, placed and presented to the inspection camera on the right. Finally, both bearings are discarded on the conveyor belt on the left after two handovers.
	
	\section{Conclusion}
	
	In this paper, we have presented a depth-first logic tree search for multi-goal deterministic decision processes based on PNE and AR. It is able to quickly identify long action sequences in the presence of a high number of sub-goals. We showed that its computational complexity corresponds to the one of solving each sub-goal problem individually while avoiding the additional implementation burden. We demonstrated in a real-robot scenario that these logic action sequences can be indeed realized by two robots in a bearing inspection scenario. 
	
	In future work, we would like to explore the possibility of an automated selection of the bridging factor for example based on machine learning methods similar to~\cite{Davidov2006}.
	
	\balance
	\bibliographystyle{IEEEtran}
	\bibliography{bib.bib}

\begin{thebibliography}{10}
\providecommand{\url}[1]{#1}
\csname url@samestyle\endcsname
\providecommand{\newblock}{\relax}
\providecommand{\bibinfo}[2]{#2}
\providecommand{\BIBentrySTDinterwordspacing}{\spaceskip=0pt\relax}
\providecommand{\BIBentryALTinterwordstretchfactor}{4}
\providecommand{\BIBentryALTinterwordspacing}{\spaceskip=\fontdimen2\font plus
\BIBentryALTinterwordstretchfactor\fontdimen3\font minus
  \fontdimen4\font\relax}
\providecommand{\BIBforeignlanguage}[2]{{%
\expandafter\ifx\csname l@#1\endcsname\relax
\typeout{** WARNING: IEEEtran.bst: No hyphenation pattern has been}%
\typeout{** loaded for the language `#1'. Using the pattern for}%
\typeout{** the default language instead.}%
\else
\language=\csname l@#1\endcsname
\fi
#2}}
\providecommand{\BIBdecl}{\relax}
\BIBdecl

\bibitem{Huy2020}
H.~Nguyen, N.~Adrian, J.~L. Xin~Yan, J.~M. Salfity, W.~Allen, and Q.-C. Pham,
  ``Development of a robotic system for automated decaking of 3d-printed
  parts,'' in \emph{2020 IEEE International Conference on Robotics and
  Automation (ICRA)}, 2020, pp. 8202--8208.

\bibitem{Pfeiffer2017}
\BIBentryALTinterwordspacing
K.~Pfeiffer, A.~Escande, and A.~Kheddar, ``Nut fastening with a humanoid
  robot,'' in \emph{2017 IEEE/RSJ International Conference on Intelligent
  Robots and Systems (IROS)}.\hskip 1em plus 0.5em minus 0.4em\relax IEEE
  Press, 2017, p. 6142–6148. [Online]. Available:
  \url{https://doi.org/10.1109/IROS.2017.8206515}
\BIBentrySTDinterwordspacing

\bibitem{Lim2021}
\BIBentryALTinterwordspacing
J.~H. Lim, X.~Zhang, G.~H.~A. Ting, and Q.-C. Pham, ``Stress-cognizant 3d
  printing of free-form concrete structures,'' \emph{Journal of Building
  Engineering}, vol.~39, p. 102221, 2021. [Online]. Available:
  \url{https://www.sciencedirect.com/science/article/pii/S2352710221000772}
\BIBentrySTDinterwordspacing

\bibitem{Coulom206}
R.~Coulom, ``Efficient selectivity and backup operators in monte-carlo tree
  search,'' vol. 4630, 05 2006.

\bibitem{strips1971}
\BIBentryALTinterwordspacing
R.~E. Fikes and N.~J. Nilsson, ``Strips: A new approach to the application of
  theorem proving to problem solving,'' \emph{Artificial Intelligence}, vol.~2,
  no.~3, pp. 189--208, 1971. [Online]. Available:
  \url{https://www.sciencedirect.com/science/article/pii/0004370271900105}
\BIBentrySTDinterwordspacing

\bibitem{Ren2021}
T.~Ren, G.~Chalvatzaki, and J.~Peters, ``Extended task and motion planning of
  long-horizon robot manipulation,'' 03 2021.

\bibitem{Kocsis2006}
L.~Kocsis and C.~Szepesvári, ``Bandit based monte-carlo planning,'' vol. 2006,
  09 2006, pp. 282--293.

\bibitem{gelly2011}
\BIBentryALTinterwordspacing
S.~Gelly and D.~Silver, ``Monte-carlo tree search and rapid action value
  estimation in computer go,'' \emph{Artificial Intelligence}, vol. 175,
  no.~11, pp. 1856--1875, 2011. [Online]. Available:
  \url{https://www.sciencedirect.com/science/article/pii/S000437021100052X}
\BIBentrySTDinterwordspacing

\bibitem{Lertkultanon2015}
P.~Lertkultanon and Q.-C. Pham, ``A single-query manipulation planner,''
  \emph{IEEE Robotics and Automation Letters}, vol.~1, no.~1, pp. 198--205,
  2016.

\bibitem{Driess2021}
\BIBentryALTinterwordspacing
D.~Driess, J.-S. Ha, and M.~Toussaint, ``Learning to solve sequential physical
  reasoning problems from a scene image,'' \emph{The International Journal of
  Robotics Research}, vol.~40, no. 12-14, pp. 1435--1466, 2021. [Online].
  Available: \url{https://doi.org/10.1177/02783649211056967}
\BIBentrySTDinterwordspacing

\bibitem{Davidov2006}
D.~Davidov and S.~Markovitch, ``Multiple-goal heuristic search,'' \emph{J.
  Artif. Intell. Res. (JAIR)}, vol.~26, pp. 417--451, 08 2006.

\bibitem{Henrich1998}
D.~Henrich, C.~Wurll, and H.~W{\"o}rn, ``Multi-directional search with goal
  switching for robot path planning,'' in \emph{Tasks and Methods in Applied
  Artificial Intelligence}, A.~Pasqual~del Pobil, J.~Mira, and M.~Ali,
  Eds.\hskip 1em plus 0.5em minus 0.4em\relax Berlin, Heidelberg: Springer
  Berlin Heidelberg, 1998, pp. 75--84.

\bibitem{Matsuta2010}
K.~Matsuta, H.~Kobayashi, and A.~Shinohara, ``Multi-target adaptive a,'' in
  \emph{Adaptive Agents and Multi-Agent Systems}, 2010.

\bibitem{Ishida2019}
S.~Ishida, M.~Rigter, and N.~Hawes, ``Robot path planning for multiple target
  regions,'' 09 2019, pp. 1--6.

\bibitem{Vien2015}
\BIBentryALTinterwordspacing
N.~A. Vien and M.~Toussaint, ``Hierarchical monte-carlo planning,'' in
  \emph{AAAI}, 2015, pp. 3613--3619. [Online]. Available:
  \url{http://www.aaai.org/ocs/index.php/AAAI/AAAI15/paper/view/9608}
\BIBentrySTDinterwordspacing

\bibitem{Waard2016}
M.~de~Waard, D.~M. Roijers, and S.~C. Bakkes, ``Monte carlo tree search with
  options for general video game playing,'' in \emph{2016 IEEE Conference on
  Computational Intelligence and Games (CIG)}, 2016, pp. 1--8.

\bibitem{Baier2012}
H.~Baier and M.~Winands, ``Beam monte-carlo tree search,'' 09 2012, pp.
  227--233.

\bibitem{pddl1998}
M.~Ghallab, C.~Knoblock, D.~Wilkins, A.~Barrett, D.~Christianson, M.~Friedman,
  C.~Kwok, K.~Golden, S.~Penberthy, D.~Smith, Y.~Sun, and D.~Weld, ``Pddl - the
  planning domain definition language,'' 08 1998.

\bibitem{diankov2010automated}
R.~Diankov, ``Automated construction of robotic manipulation programs,'' 2010.

\bibitem{Toussaint2015}
M.~Toussaint, ``Logic-geometric programming: An optimization-based approach to
  combined task and motion planning,'' in \emph{Proceedings of the 24th
  International Conference on Artificial Intelligence}, ser. IJCAI'15.\hskip
  1em plus 0.5em minus 0.4em\relax AAAI Press, 2015, p. 1930–1936.

\end{thebibliography}
	
\end{document}